\documentclass[letterpaper]{article} 
\usepackage{aaai23}  
\usepackage{times}  
\usepackage{helvet}  
\usepackage{courier}  
\usepackage[hyphens]{url}  
\usepackage{graphicx} 
\usepackage{enumitem}
\usepackage{natbib}  
\usepackage{caption} 
\usepackage{algorithm}

\newcommand{\mD} {\ensuremath{\mathcal{D}}}

\newcommand{\mN} {\ensuremath{\mathcal{N}}}
\newcommand{\mO} {\ensuremath{\mathcal{O}}}

\newcommand{\R} {\mathbb{R}}
\newcommand{\N} {\mathbb{N}}

\newcommand{\wtilde} {\widetilde}

\newcommand{\floor}[1]{\lfloor #1 \rfloor}
\newcommand{\ceil}[1]{\lceil #1 \rceil}
\newcommand{\expct}[1]{\mathbb{E}\left[#1\right]}
\newcommand{\expcthat}[1]{\widehat{\mathbb{E}}\left[#1\right]}
\newcommand{\expctu}[2]{\mathbb{E}_{#1}\left[#2\right]}
\newcommand{\norm}[1]{\left\lVert#1\right\rVert}

\newcommand{\ind}[1]{\mathds{1}\left[#1\right]}

\newcommand{\bX} {\ensuremath{\mathbf{X}}}
\newcommand{\bZ} {\ensuremath{\mathbf{Z}}}
\newcommand{\dimp}{\textnormal{imp}}
\newcommand{\varHat}{\widehat{\text{Var}}}
\newcommand{\unif} {\textnormal{Unif}}

\newcommand{\dstump} {\textsc{DStump}~}
\newcommand{\lasso} {\textsc{Lasso}~}
\newcommand{\sis} {\textsc{SIS}~}

\newcommand{\eat}[1]{}

\newcommand{\oge}{\gtrsim}
\newcommand{\treerank}{\textsc{StumpScore}}

\usepackage{url}            
\usepackage{booktabs}       
\usepackage{amsfonts, dsfont}       
\usepackage{nicefrac}       
\usepackage{microtype}      
\usepackage{xcolor}         

\usepackage{caption}
\usepackage{subcaption}
\usepackage{amsmath}
\usepackage{amssymb}
\usepackage{mathtools}
\usepackage{amsthm}
\usepackage{enumitem}

\usepackage[noend]{algpseudocode}

\theoremstyle{plain}
\newtheorem{theorem}{Theorem}[section]

\newtheorem{lemma}[theorem]{Lemma}
\newtheorem{corollary}[theorem]{Corollary}
\theoremstyle{definition}
\newtheorem{definition}[theorem]{Definition}

\theoremstyle{remark}

\algnewcommand\algorithmicinput{\textbf{Input:}}
\algnewcommand\algorithmicoutput{\textbf{Output:}}
\algnewcommand\Input{\item[\algorithmicinput]}%
\algnewcommand\Output{\item[\algorithmicoutput]}%

\usepackage{newfloat}
\usepackage{listings}
\DeclareCaptionStyle{ruled}{labelfont=normalfont,labelsep=colon,strut=off} 
\floatstyle{ruled}
\newfloat{listing}{tb}{lst}{}
\floatname{listing}{Listing}
\title{Optimal Sparse Recovery with Decision Stumps}
\author{
    Kiarash Banihashem,
    MohammadTaghi Hajiaghayi,
    Max Springer
}
\affiliations{
    University of Maryland\\
    kiarash@umd.edu,
    hajiagha@cs.umd.edu,
    mss423@umd.edu
}
\usepackage{bibentry}

\newif\ifappendix
\appendixtrue

\ifappendix
\usepackage{wrapfig}
\usepackage{hyperref}
\nocopyright
\fi

\begin{document}
\maketitle

\begin{abstract}
  Decision trees are widely used for their low computational cost, good
  predictive performance, and ability to assess the importance of features.
  Though often used in practice for feature selection, the theoretical
  guarantees of these methods are not well understood. We here obtain a tight
  finite sample bound for the feature selection problem in linear regression
  using single-depth decision trees. We examine the statistical properties of
  these ``decision stumps" for the recovery of the $s$ active features from $p$
  total features, where $s \ll p$. Our analysis provides tight sample performance guarantees on
  high-dimensional sparse systems which align with the finite sample bound of
  $O(s \log p)$ as obtained by Lasso, improving upon previous bounds for both
  the median and optimal splitting criteria. Our results extend to the
  non-linear regime as well as arbitrary sub-Gaussian distributions,
  demonstrating that tree based methods attain strong feature selection
  properties under a wide variety of settings and further shedding light on the
  success of these methods in practice. As a byproduct of our analysis, we show
  that we can provably guarantee recovery even when the number of active
  features $s$ is unknown.
  We further validate our theoretical results and proof methodology
  using computational experiments.
\end{abstract}

\section{Introduction}
\label{intro}
  Decision trees are one of the most popular tools used in machine learning.
  Due to their simplicity and interpretability,
  trees are widely implemented 
  by data scientist,
  both individually,
  and in aggregation
  with ensemble methods
  such as random forests and
  gradient boosting \cite{Friedman2001}.
  
  In addition to their predictive accuracy,
  tree based methods are an important tool
  used for the \emph{variable selection}
  problem: 
  identifying a relevant small subset of a high-dimensional feature space of
  the input variables that can accurately predict the output.
  When the relationship between the variables
  is linear,
  it has long been known that \lasso
  achieves the optimal sample complexity rate for this problem \cite{Wainwright2009_1}.
  In practice, however,
  tree-based methods have been shown to be preferable
  as they scale linearly with the size of the data and can capture non-linear relationships between the
  variables \cite{xu2014gradient}.
  
  Notably, various tree structured systems are implemented for this variable selection task across wide-spanning domains. For example, tree based importance measures have been used for prediction of financial distress \cite{Qian_2022}, wireless signal recognition \cite{Li_2018}, credit scoring \cite{Xia_2017},  and the discovery of genes in the field of bioinformatics \cite{Breiman_2001,Bureau_2005,HuynhThu_2010,Lunetta_2004} to name a small fraction.

  Despite this empirical success however,
  the theoretical properties of these tree based methods for feature selection
  are not well understood.
  While several papers
  have considered variants of this problem
  (see Section \ref{related_work} for an overview),
  even in the simple linear case,
  the sample complexity of the decision tree is not well-characterized.
  
  In this paper, we attempt to bridge this gap and analyze
  the variable selection properties of
  single-level decision trees, commonly referred to as ``decision stumps" (\dstump).
  Considering both linear and non-linear settings,
  we show that \dstump achieves nearly-tight
  sample complexity rates for a variety of practical sample distributions. Compared to prior work, our analysis is simpler and
  applies to
  different variants of
  the decision tree,
  as well as 
  more general function classes.
  
  The remainder of the paper is organized as follows: in Section
  \ref{our_results} we introduce our main results on the finite sample
  guarantees of \dstump, in Section \ref{related_work} we discuss important
  prior results in the field of sparse recovery and where they fall flat, in
  Section \ref{alg_desc} we present the recovery algorithm,
  and in Section \ref{main_result},
  we provide
  theoretical guarantees for the procedure under progressively more
  general problem assumptions.
  The proofs of these results provided in Section \ref{proofs}. We supplement the theoretical
  results with computational experiments in Section \ref{exp_results} and
  finally conclude the paper in Section \ref{conclusion}. 
  \vspace{-1mm}
\section{Our Results}
\label{our_results}
  We assume that we are given
  a dataset $\mD = \{(\bX_{i, :}, Y_i)\}_{i=1}^n$
  consisting of $n$ samples
  from the \emph{non-parametric regression} model
  $Y_i = \sum_{k} f_k(\bX_{i, k}) + W_i$
  where $i$ denotes the sample number,
  $\bX_{i, :} \in \R^p$ is the input vector with corresponding output
  $Y_i\in \R$, $W_i\in \R$ are
  i.i.d noise values and
  $f_k: \R \to \R$ are univariate functions that are \emph{$s$-sparse}: 
  \begin{definition}
    A set of univariate functions $\{f_k\}_{k\in [p]}$ is $s$-\emph{sparse} on feature set $[p]$ if there exists a set $S \subseteq [p]$ with size $s = |S| \ll p$ such that $$f_k = 0 \iff k \notin S$$
  \end{definition}
  \noindent Given access to $(\bX_{i, :}, Y_i)_{i=1}^n$,
  we consider the \emph{sparse recovery}
  problem, where we attempt to reveal the set $S$ with only a minimal number of samples.
  Note that this is different from the \emph{prediction}
  problem where the goal is to learn the functions $f_k$.
  In accordance with prior work \cite{Han2020,Kazemitabar2017, Klusowski2020}
  we assume the $\bX_{i, j}$ are i.i.d draws from the
  uniform distribution on $[0, 1]$ with Gaussian noise, 
  $W_i \sim \mathcal{N}(0, 1)$.
  In Section \ref{main_result}, we will 
  discuss how this assumption can be relaxed using
  our non-parametric results to consider more
  general distributions.

  For the recovery problem,
  we consider the \dstump
  algorithm
  as first proposed by \cite{Kazemitabar2017}.
  The algorithm, shown in Algorithm \ref{alg:stump},
  fits a single-level decision tree (stump)
  on each feature 
  using either the ``median split"
  or the ``optimal split" strategy
  and ranks the features by the error
  of the corresponding trees. The median split provides a substantially simplified implementation as we do not need to optimize the stump construction, further providing an improved run time over the comparable optimal splitting procedure. 
  Indeed, the median and optimal split have time complexity at most
  $\mO(np)$ and $\mO(np\log(n))$ respectively.
  
  In spite of this simplification, we show that
  in the widely considered case of \emph{linear design}, where the $f_k$ are linear, \dstump can correctly recover
  $S$ with a sample complexity bound of
  $\mO(s\log p)$, matching 
  the minimax optimal lower bound for the problem as
  achieved by \lasso \cite{Wainwright2009_1,Wainwright2009_2}.
  This result is noteworthy and surprising since the \dstump algorithm (and decision trees in general) is not designed with a linearity assumption, as is the case with \lasso. 
  For this reason, trees are in general utilized for their predictive power in a non-linear model, yet our work proves their value in the opposite.
  We further extend these results for non-linear models and general 
  sub-Gaussian distributions, improving previously known results
  using simpler analysis. 
  In addition, our results do not require the sparsity level $s$ to be known in advance.
  We summarize our main technical results as follows:
    \begin{itemize}
    \item
    We obtain a sample complexity bound of $\mO(s\log p)$
    for the \dstump
    algorithm
    in the linear design
    case,
    matching the optimal rate
    of \lasso and improving prior bounds in the literature for both
    the median and optimal split. This is the first \emph{tight} bound on the sample complexity of single-depth decision trees used for sparse recovery and significantly improves on the existing results.
    \item
    We extend 
    our results to the case of non-linear $f_k$,
    obtaining tighter results
    for a wider class of functions compared to the existing literature.
    We further use these results to
    generalize our analysis to
    sub-Gaussian distributions via the extension to nonlinear $f_k$.
    \item As a byproduct of our improved analysis, we show that our results
      hold for the case where the number of active features $s$ is not known.
      This is the first theoretical guarantee on decision stumps that does
      not require $s$ to be known.
    \item
    We validate our theoretical results using numerical simulations that show the necessity of
    our analytic techniques.
  \end{itemize}
\section{Related Work}
\label{related_work}
While our model framework and the sparsity problem as a whole have been studied extensively 
\cite{Fan2006,Lafferty2008,Wainwright2009_1,Wainwright2009_2}, none have replicated the well 
known optimal lower bound for the classification problem under the given set of
assumptions. Our work provides improved finite sample
guarantees on \dstump for the regression problem that nearly match that of
\lasso by means of weak learners.

Most closely related to our work is that of \cite{Kazemitabar2017}, which formulates the \dstump algorithm and theoretical approach for finite sample guarantees of
these weak learners.
Unlike our nearly tight result on the number of samples required for recovery, \cite{Kazemitabar2017} provides a weaker $\mO(s^2 \log p)$ bound when
using the \emph{median splitting} criterion.
We here demonstrate that the procedure can obtain the near optimal finite sample guarantees by highlighting a 
subtle nuance in the analysis of the stump splitting (potentially of independent interest to the reader);
instead of using the variance of one sub-tree as an impurity measure,
we use the variance of both sub-trees.
As we will show both theoretically and experimentally, this consideration
is vital for obtaining tight bounds.
Our analysis is also more general than that of \cite{Kazemitabar2017}, with applications to both 
median and optimal splits, a wider class of functions $f_k$, and more general distributions.

In a recent work, \cite{Klusowski2021} provide an indirect analysis of
the
\dstump formulation with the \emph{optimal splitting} criterion, by studying its relation to the \sis algorithm of \cite{Fan2006}.
Designed for linear models specifically, 
\sis sorts the features based on their Pearson correlation with the output,
and has optimal sample complexity for the linear setting.
\cite{Klusowski2021} show that when using the optimal splitting criterion, 
 \dstump is \emph{equivalent} to \sis up to logarithmic factors, 
which leads to a sample complexity of $$n \oge s \log(p) \cdot \left(
\log(s) + \log(\log(p))
\right).$$
This improves the results of \cite{Kazemitabar2017}, though the analysis
is more involved.
A similar technique is also used to study the non-linear case,
but the conditions assumed on $f_k$ are hard to interpret and the bounds are weakened.
Indeed, instantiating the non-parametric results for the linear case
implies a \emph{sub-optimal} sample complexity rate of $\mO(s^2\log p)$.
In addition, \cite{Klusowski2021} describe a heuristic algorithm for the case of \emph{unknown} $|S|$, though they fail to prove any guarantees on its performance.

In contrast, we
provide a direct analysis of \dstump.
This allows us to obtain optimal bounds for \emph{both} the \emph{median} and
\emph{optimal} split in the linear case, as well as improved and generalized
bounds in the non-linear case. Our novel approach further allows us to analyze
the case of unknown sparsity level, as we provide the first formal proof for
the heuristic algorithm suggested in \cite{Klusowski2021} and
\cite{fan2011nonparametric}. 
Compared to prior work, our analysis is considerably simpler and applies to
more general settings.

Additionally, various studies have leveraged the simplicity of median splits in
decision trees to produce sharp bounds on mean-squared error for the regression
problem with \emph{random forests} \cite{Duroux_2018,Klusowski_21b}.
In each of these studies, analysis under the median split assumption allows for
improvements in both asymptotic and finite sample bounds on prediction error
for these ensemble weak learners. 
In the present work, we extend this intuition to the feature selection problem for high-dimensional sparse systems, and
further emphasize the utility of the median split even in the singular decision
stump case.

\section{Algorithm Description} \label{alg_desc}
  \subsection{Notation and Problem Setup}
    For mathematical convenience, we adopt matrix notation and use
    $\bX = (\bX_{ij})\in \R^{n \times p}$ and $Y = (Y_i) \in \R^{n}$
    to denote the input matrix and the output vector respectively.
    We use $\bX_{i,:}$ and $X^k$ to refer the $i$-th row and
    $k$-th column of the matrix $\bX$.
    We will also extend the definition of the univariate functions
    $f_k$ to vectors by assuming that
    the function is applied to each coordinate separately:
    for $v\in \R^d$, we define $f_k(v)\in \R^d$ as
    $\left(f_k(v)\right)_{i} = f_k(v_i)$.

    We let $\expct{\cdot}$ and $\text{Var}(\cdot)$
    denote the expectation and variance 
    for random variables, with
    $\expcthat{\cdot}$ and $\varHat\left(\cdot\right)$
    denoting their empirical counterparts i.e,
    for a generic vector $v\in \R^d$,
    $$\expcthat{v} = \frac{\sum_{i=1}^d{v_i}}{d} \quad \text{and}
    \quad
    \varHat\left(v\right) = \frac{\sum_{i=1}^d(v_i - \expcthat{v})^2}{d}.$$
    We will also use $[d]$ 
    to denote the set $\{1, \dots, d\}$.
    Finally, we let $\unif(a, b)$ denote the uniform distribution over $[a, b]$
    and use $C, c > 0$ to denote generic universal constants.

    Throughout the paper, we will use the concept of \emph{sub-Gaussian}
    random variables for stating and proving our results.
    A random variable $Z\in \R$ is called sub-Gaussian if
    there exists a $t$ for which
    $\expct{e^{(Z/t)^2}} \le 2$ and
    its 
    \emph{sub-Gaussian norm} is defined as
    $$\norm{Z}_{\psi_2} = \inf \left\{t > 0: \expct{e^{(Z/t)^2}} \le 2 \right\}.$$
    Sub-Gaussian random variables 
    are well-studied and have many desirable properties, 
    (see \cite{Vershynin2018} for a comprehensive overview),
    some of which we outline below as they are leveraged throughout our analysis.
    
    \begin{enumerate}[label=(\textbf{P{{\arabic*}}}),leftmargin=*]
        \item (Hoeffding's Inequality) If $Z$ is a sub-Gaussian random variable, then for any $t > 0$,
        $$ \Pr \left(\left|{Z - \expct{Z}}\right| \ge t\right)  \le 2e^{-c \left({t} / {\norm{Z}_{\psi_2}}\right)^2}.$$ \label{P1}
        \item If $Z_1, \dots Z_n$ are a sequence of independent sub-Gaussian random variables, then $\sum Z_i$ is also sub-Gaussian with norm satisfying $$\norm{\sum Z_i}_{\psi_2}^2 \le \sum \norm{Z_i}_{\psi_2}^2.$$ \label{P2}
        \item If $Z$ is a sub-Gaussian random variable, then so is $Z- \expct{Z}$ and $\norm{Z - \expct{Z}}_{\psi_2} \le c \norm{Z}_{\psi_2}$. \label{P3}
    \end{enumerate}

  \subsection{\dstump Algorithm}
    We here present the recovery algorithm
    \textsc{DStump}. For each feature $k\in [p]$,
    the algorithms fits a single-level decision tree or ``stump''
    on the given feature and defines the impurity of the feature
    as the error of this decision tree.
    Intuitively, the active features are expected to
    be better predictors of $Y$ and therefore have \emph{lower} impurity values.
    Thus, the algorithm sorts the features based on these values,
    and outputs the $|S|$ features with smallest impurity.
    A formal pseudo-code of our approach is given in Algorithm \ref{alg:stump}.
    
        \begin{algorithm}[H]
      \caption{Scoring using \dstump}
      \label{alg:stump}
      \begin{algorithmic}[1]
        \Input $\bX\in \R^{n\times p}, Y\in \R^n, s\in \N$
          \Output Estimate of $S$
        \For{$k\in \{1, ..., p\}$}
          \State $\tau^k = \text{argsort}(X^k)$
          \For{ $n_{L} \in \{1, \dots, n\}$ }
          \State $n_{R} = n - n_L$
            \State $Y^k_{L}  = (Y_{\tau^k(1)}, ... Y_{\tau^k(n_L)})^T$
            \State $Y^k_{R} = (Y_{\tau^k(n_L+1)}, ... Y_{\tau^k(n)})^T $
            \State $\dimp_{k, n_{L}} = \frac{n_L}{n}\varHat(Y^k_{L}) + \frac{n_R}{n}\varHat(Y^k_{R})$\label{line:imp_1}
          \EndFor
          \If{\emph{median split}}
            \State $\dimp_{k} = \dimp_{k, \floor{\frac{n}{2}}}$.
            \label{line:imp_2}
          \Else
            \State $\dimp_{k} = \min_{\ell}\dimp_{k, \ell}$.
            \label{line:imp_3}
          \EndIf
        \EndFor
        \State \Return $\tau = \text{argsort(\dimp)}$ \label{line:tau}
      \end{algorithmic}
    \end{algorithm}
    \noindent Formally, for each feature $k$,
    the $k$-th column is sorted in increasing order such that
    $X^k_{\tau^k(1)}\le X^k_{\tau^k(2)} \dots \le X^k_{\tau^k(n)}$ with ties broken randomly.
    The samples are then split into two groups:
    the \emph{left half}, consisting of
    $X^k_{i} \le X^k_{\tau^k(n_{L})}$ and
    the \emph{right half} consisting of
    $X^k_{i} > X^k_{\tau^k(n_{L})}$.
    Conceptually, this is the same as splitting a single-depth tree on
    the $k$-th column with a $n_L$ to $n_R$ ratio and collecting the samples in each group.
    The algorithm then calculates the
    variance of the output in each group,
    which represents the optimal prediction error for
    this group with a single value as
    $\varHat\left({Y^k_{L}}\right) = \min_{t}\frac{1}{n_L} \cdot \sum (Y^k_{L, i} - t)^2$.
    The average of these two variances is taken as the impurity. Formally,
    \begin{align}
        \dimp_{k, n_L} =
         \frac{n_L}{n}\varHat(Y^k_{L}) +
        \frac{n_R}{n}\varHat(Y^k_{R}).
        \label{eq:dimp_regr}
    \end{align}
    For the \emph{median split} algorithm,
    the value of $n_{L}$ is simply chosen as $\floor{\frac{n}{2}}$ or $\ceil{\frac{n}{2}}$,
    where as for the optimal split, the value
    is chosen in order to minimize the impurity of the split.
    Lastly, the features are sorted by their impurity values and
    the $|S|$ features with lowest impurity are predicted as
    $S$.
    In Section \ref{sec:unk}, we discuss the case of unknown $|S|$ and obtain an algorithm with nearly matching guarantees.
    
\section{Main Results}\label{main_result}
  \subsection{Linear Design} \label{sec_linear}
    For our first results, we consider the simplest setting of linear models with
    uniform distribution of the inputs.
    Formally, we assume that there is a vector
    $\beta_k \in \R^p$ such that
    $f_k(x) = \beta_k\cdot x$ for all $k$.
    This is equivalent to considering the \emph{linear regression}
    model
    $Y = \sum_k \beta_k X^k + W$.
    We further assume that each entry of the matrix
    $\bX$ is an i.i.d draw from the uniform distribution on
    $[0, 1]$. This basic setting is important from
    a theoretical perspective as it allows us to compare
    with existing results 
    from the literature before extending to more general contexts. This initial result of Theorem \ref{thm_linear},
    provides an upper bound on the failure probability for \dstump in this setting.
    \begin{theorem}\label{thm_linear}
      Assume that each entry of the input matrix
      $X$ is sampled i.i.d from a uniform distribution on $[0, 1]$.
      Assume further that the output vectors satisfy
      the \emph{linear regression} model
      $Y = \sum_k \beta_k X^k + W$
      where $W_i$ are sampled i.i.d from
      $N(0, \sigma_{w}^2)$.
      Algorithm \ref{alg:stump} correctly recovers the active feature set $S$ with probability at least
      \begin{align*}
        1 
         - 4s e^{-c\cdot n}
        - 4pe^{-c \cdot n\frac{\min_{k}\beta_k^2}{\norm{\beta}_2^2 + \sigma_{w}^2}}.
      \end{align*}
      for the median split,
      and with probability at least
      \begin{align*}
        1 
         - 4s e^{-c\cdot n}
        - 4npe^{-c \cdot n\frac{\min_{k}\beta_k^2}{\norm{\beta}_2^2 + \sigma_{w}^2}}.
      \end{align*}
      for the optimal split.
    \end{theorem}
    
    \noindent Moreover, the above theorem provides a finite sample guarantee
    for the $\dstump$ algorithm
    and does not make \emph{any} assumptions on the
    parameters or their asymptotic relationship.
    In order to obtain a comparison with existing
    literature, \cite{Kazemitabar2017, Klusowski2021, Wainwright2009_1, Wainwright2009_2},
    we consider these bounds in the asymptotic regime
    $\min_{k \in S} \beta_k^2 \in \Omega(\frac{1}{s})$
    and $\norm{\beta}_2 \in \mO(1)$.
    \begin{corollary}\label{cor.only}
      In the same setting as Theorem \ref{thm_linear},
      assume  that $\norm{\beta}_2^2 \in \mO(1)$
      and $\min_{k\in S} \beta_k^2 \in \Omega(\frac{1}{s})$.
      Then Algorithm \ref{alg:stump} correctly recovers the active feature set $S$
      with high probability as long as $n \gtrsim s\log p$.
    \end{corollary}
    \noindent The proof of the Corollary is presented in Appendix.
    The above result shows that \dstump is optimal
    for recovery when the data obeys a linear model.
    This is interesting considering tree based methods are known for their strength in capturing non-linear relationships and are not designed with a linearity assumption like \textsc{Lasso}.
    In the next section, we further extend our finite sample bound analysis to
    \emph{non-linear models}.
  \subsection{Additive Design} \label{sec:additive}
    We here consider the case of non-linear $f_k$
    and obtain theoretical guarantees for the original \dstump algorithm.
    Our main result is Theorem \ref{thm_additive_monotonic}
    stated below.
    \begin{theorem}\label{thm_additive_monotonic}
      Assume that each entry of the input matrix
      $X$ is sampled i.i.d from a uniform distribution on $[0, 1]$
      and $Y = \sum_k f_k(X^k) + W$
      where $W_i$ are
      sampled i.i.d from $\mathcal{N}(0, \sigma_{w}^2)$.
      Assume further that each $f_k$ is monotone and 
      $f_k(\unif(0, 1))$ is sub-Gaussian with sub-Gaussian norm
      $\norm{f_k(\unif(0, 1))}_{\psi_2}^2 \le \sigma_{k}^2$.
      For $k\in S$, define $g_k$ as
      \begin{align}
        g_k := \expct{f_k(\unif(\frac{1}{2}, 1))} - \expct{f_k(\unif(0, \frac{1}{2}))}
        \label{eq:g_k}
      \end{align}
      and define $\sigma^2$ as $\sigma^2 = \sigma_w^2 + \sum_k \sigma_k^2$.
      Algorithm \ref{alg:stump} correctly recovers the set $S$ with probability at least
    $$
        1 - 4se^{-cn} - 4pe^{-cn\frac{\min_{k} g_k^2}{\sigma^2}}
   $$
      for the median split and  
      $$1 - 4se^{-cn} - 4npe^{-cn\frac{\min_{k} g_k^2}{\sigma^2}}$$ for the optimal split.
    \end{theorem}
    \noindent Note that, 
    by instantiating Theorem \ref{thm_additive_monotonic}
    for linear models, we obtain the same bound as in Theorem \ref{thm_linear} implying the above bounds are tight in the linear setting.

    The extension to all monotone functions 
    in Theorem \ref{thm_linear}
    has an important theoretical consequence:
    since the \dstump algorithm is invariant under monotone transformations of
    the input, we can obtain the same results for \emph{any} distribution of
    $\bX_{i,:}$.
    As a simple example, consider $\bX_{ij} \sim \mN(0, 1)$ and assume that we are interested in bounds
    for linear models.
    Define the matrix $\bZ$ as $\bZ_{ij} = F_{\mN}(\bX_{ij})$ where $F_{\mN}(.)$
    denotes the CDF of the Gaussian distribution.
    Since the CDF is an increasing function, running the \dstump
    algorithm with $(\bZ, Y)$ produces the same output as running it with
    $(\bX, Y)$.
    Furthermore, applying the CDF of
    a random variable to itself yields a uniform
    random variable.
    Therefore, $\bZ_{ij}$ are i.i.d draws of the
    $\text{Unif}(0, 1)$ distribution.
    Setting $f_k(t)= \beta_k \cdot F_{\mN}^{-1}(t)$, the results of Theorem \ref{thm_additive_monotonic}
    for $(\bZ, Y)$ imply the same bound as Theorem \ref{thm_linear}.
    Notably, we can obtain the same sample complexity bound of
    $\mO(s\log p)$ for the Gaussian distribution as well.
    In the appendix, 
    we discuss a generalization of this idea,
    which effectively allows us to remove the uniform distribution condition
    in Theorem \ref{thm_additive_monotonic}.
\subsection{Unknown sparsity level}
\label{sec:unk}
A drawback of the previous results are that they assume $|S|$ is given when, in general, this is not the case.
Even if $|S|$ is not known however, Theorem \ref{thm_additive_monotonic} guarantees that the active features are ranked higher than non-active ones
in $\tau$, i.e, 
$\tau(k) < \tau(k')$ for all $k\in  S, k' \notin S$. In order to recover $S$, it suffices to find a threshold $\gamma$ such that
$\max_{k\in S} \dimp_k \le \gamma \le \min_{k\notin S} \dimp_k$.

To solve this, we use the so called ``permutation algorithm'' which is a well known heuristic in the statistics literature \cite{barber2015controlling, chung2013exact, chung2016multivariate,fan2011nonparametric} and was discussed (without proof) by \cite{Klusowski2021} as well.
Formally, we apply a random permutation $\sigma$ on the rows of $X$,
obtaining the matrix $\sigma(X)$ where
$\sigma(X)_{ij} = X_{\sigma(i), j}$,
We then rerun Algorithm \ref{alg:stump} with $\sigma(X)$ and $Y$ as input.
The random permutation
means that $X$ and $Y$ were ``decoupled'' from each other
and
effectively, \emph{all} of the features are now inactive.
We therefore expect $\min_{i, t} \dimp_i(\sigma(X), y)$ to be close to
$\min_{k \notin S} \dimp_k(X, y)$.
Since this estimate may be somewhat conservative, we repeat the sampling and take the minimum value across these repetitions. 
A formal pseudocode is provided in Algorithm \ref{alg:unkn}.
The \treerank{} method is the same algorithm as Algorithm \ref{alg:stump}, with the distinction that it returns 
$\dimp$ in Line \ref{line:tau}

\begin{algorithm}[H]
      \caption{Unknown $|S|$}
      \label{alg:unkn}
      \begin{algorithmic}[1]
        \Input $\bX\in \R^{n\times p}, Y\in \R^n$
        \Output
        {$\gamma\in [\max_{k\in S} \dimp_k ,\min_{k\notin S} \dimp_k]$}
          \For{$t \gets 1, \dots, T - 1$} 
            \State $\sigma^{(t)} \gets$ Random permutation on $[n]$.
            \label{line:random_permute}
            \State $\dimp^{(t)} \gets \treerank{}(\sigma^{(t)}(X), y) $
          \EndFor
          \Return $\min_{i, t} \dimp_i^{(t)}$
      \end{algorithmic}
\end{algorithm}

\noindent Assuming we have used $T$ repetitions in the algorithm, the probability that
$\min_{k \notin S} \dimp_k(X, y) \le \gamma$ is at most
$\frac{1}{T}$. 
While we provide a formal proof in the appendix, the main intuition behind the result is that
$\dimp_{i}^{(t)}$ and $\dimp_{k'}$ for $k'\notin S$ are all the impurities corresponding to inactive features. Thus,
the probability that the maximum across all of these impurities falls in $[p] \backslash S$ is at most
$\frac{p-s}{Tp} \le \frac{1}{T}$. Ensuring that
$\max_{k \in S} \dimp_k(X, y) \le \gamma$ involves treating 
the extra $T$ impurity scores as $(T-1)p$ extra inactive features. This means
that we can use the same results of Theorem \ref{thm_additive_monotonic} with $p$ set
to $Tp$ since our sample complexity bound is \emph{logarithmic} in $p$. The
formal result follows with proof in the appendix.

\begin{theorem}\label{thm.unknown_s}
  In the same setting as Theorem \ref{thm_additive_monotonic},
  let $\gamma$ be the output of Algorithm \ref{alg:unkn} and
  set $\widehat{S}$ to be
  $\{ k: \dimp_k \le \gamma\}$. The probability that
  $\widehat{S}=S$ is at least 
  $$
        1 - T^{-1} - 4se^{-cn} - 4Tpe^{-cn\frac{\min_{k} g_k^2}{\sigma^2}}
    $$
    for the median split and at least
    $$
        1 - T^{-1} - 4se^{-cn} - 4nTpe^{-cn\frac{\min_{k} g_k^2}{\sigma^2}}
    $$
    for the optimal split.
\end{theorem}
\noindent We note that if we set $T = n^c$ for some constant $c > 0$, we obtain the same $\mO(s \log p)$ as before. 
\section{Proofs}
\label{proofs}
  In this section,
  we  prove Theorem \ref{thm_additive_monotonic}
  as it is the more general version of Theorem \ref{thm_linear},
  and defer with remainder of the proofs to the appendix.
  
  To prove that the algorithm succeeds,
  we need to show that
  $\dimp_k < \dimp_{k'}$
  for all $k\in S, k'\notin S$. We proceed by first proving an upper bound $\dimp_k$ for all $k \in S$ 
  \begin{lemma}\label{lemma_S_is_important}
    In the setting of Theorem \ref{thm_additive_monotonic},
    for any active feature $k\in S$,
    $\dimp_k \le \varHat\left(Y\right) - \frac{\min_{k}g_k^2}{720}$
    with probability at least
    $$
      1 - 
      4e^{-c \cdot n}
      -
      4e^{-c \cdot n \cdot \frac{\min_{k} g_k^2}{\sigma^2}}.
    $$
  \end{lemma}
  \noindent Subsequently, we need to prove
  an analogous lower bound
  on $\dimp_{k'}$
  for all $k'\notin S$.
  \begin{lemma}\label{lemma_important_is_S}
    In the setting of Theorem \ref{thm_additive_monotonic},
    for any inactive feature $k'\notin S$,
    $\dimp_{k'}
      >
      \varHat\left(Y\right) - \frac{\min_{k}g_k^2}{720}$
    with probability at least 
    $$1-4e^{-c \cdot n \cdot \frac{\min_{k} g_k^2}{\sigma^2}}$$
    for the median split and
    $$1-4ne^{-c \cdot n \cdot \frac{\min_{k} g_k^2}{\sigma^2}}$$ for the optimal split.
  \end{lemma}
  \noindent Taking the union bound over all $k,k'$,
  Lemmas \ref{lemma_S_is_important} and \ref{lemma_important_is_S}
  prove the theorem as
  they show that
  $\dimp_k < 
  \dimp_{k'}$
  for all $k\in S, k'\notin S$ with the desired probabilities.
  
  We now focus on proving Lemma \ref{lemma_S_is_important}.
  We assume without loss of generality that $f_k$ is 
  increasing \footnote{The case of decreasing $f_k$ follows by a 
  symmetric argument or by mapping $X^k \to -X^k$}.
  We further assume that $\expct{f_k(X^k_i)}=0$ as \dstump is invariant under constant shifts of the output.
  Finally, we assume that $n > 5$, as for $n \le 5$, the theorem's
  statement can be made vacuous by choosing large $c$. 
  ~\\
  We will assume $\{n_{L}, n_{R}\} = \{\floor{\frac{n}{2}}, \ceil{\frac{n}{2}}\}$ throughout the proof;
  as such our results will hold for both the optimal and median splitting criteria.
  As noted before, a key point for obtaining a tight bound is
  considering \emph{both} sub-trees in the analysis instead of considering them
  individually.
  Formally, while the impurity is usually defined via variance as in
  \eqref{eq:dimp_regr},
  it has the following equivalent definition.
  \begin{align}
    \dimp_k = 
    \varHat(Y) - 
    \frac{n_{L} \cdot n_{R}}{n^2} \cdot \left(
      \expcthat{Y^k_L} - \expcthat{Y^k_R}
    \right)^2.
    \label{mdi_identity}
  \end{align}
  The above identity is commonly
  used for the analysis of decision trees and their properties \cite{Breiman1983,Li2019, Klusowski2020, Klusowski2021}.
  From an analytic perspective,
  the key difference between \eqref{mdi_identity} and
  \eqref{eq:dimp_regr}
  is that the empirical averaging is calculated \emph{before} taking the square, allowing us to more simply analyze the \emph{first} moment rather than the second.

  Intuitively, we want to use concentration inequalities to show that
  $\expcthat{Y^k_L}$ 
  and
  $\expcthat{Y^k_R}$
  concentrate around
  their expectations and lower bound
  $|\expct{Y^k_L} - \expct{Y^k_R}|$.
  This is challenging however
  as concentration results
  typically require an i.i.d assumption
  but, as we will see, $Y^k_{L, i}$ are not i.i.d.
  More formally, for each $k$, define the random variable
  $X^k_L \in \R^{n_L}$
  as 
  $(X^k_{\tau^k(1)}, \dots, X^k_{\tau^k(n_L)})^T$ and thus
  $
    Y^k_{L, i} = f_k(X^k_{L, i}) + \sum_{j\ne k} f_j(X^j_{\tau^k(i)}) + W_{\tau^k(i)}.
    $
  While the random vectors $X^{j\ne k}$ and $W$
  have i.i.d entries,
  $X^k_{L}$
  was obtained by sorting
  the coordinates of $X^k$.
  Thus, its coordinates are non-identically distributed and dependent.
  To solve the first problem, observe that
  the empirical mean
  is invariant under permutation and
  we can thus randomly shuffle the elements of
  $Y^k_{L}$ in order to obtain a vector with
  identically distributed coordinates.
  Furthermore, by
  De Finetti's Theorem, any random vector
  with coordinates that are identically distributed (but not necessarily independent),
  can be viewed as a \emph{mixture} of i.i.d vectors, 
  effectively solving the second problem.
  Formally, the following result holds.
  \begin{lemma}[Lemma 1 in \cite{Kazemitabar2017}]\label{lemma_X_tilde_distribution}
    let $\tilde{\tau}: [n_L] \to [n_L]$
    be a random permutation on $[n_L]$ independent of
    $(\bX, W)$
    and
    define $\wtilde{X}^k_L\in \R^{n_L}$ 
    as
    $\wtilde{X}^k_{L, i} := X^k_{L, \tilde{\tau}(i)}$.
    The random vector $\wtilde{X}^k_L$ is distributed as
    a mixture of uniform i.i.d uniform vectors of size $n_L$ on $[0, \Theta]$
    with $\Theta$ sampled from $\text{Beta}(n_L + 1, n_R)$.
  \end{lemma}
  \noindent Defining 
  $\wtilde{Y}^k_L \in \R^{n_L}$
  as
  $\wtilde{Y}^k_{L, i} := Y^k_{L, \tilde{\tau}(i)}$,
  it is clear that
  $\expcthat{\wtilde{Y}^k_{L}} = \expcthat{Y^k_L}$
  and therefore we can analyze
  ${\wtilde{Y}^k_L}$
  instead of ${Y^k_L}$ as
  \begin{align*}
    \wtilde{Y}^k_{L, i} := f_k(\wtilde{X}^k_{L, i})
    + \sum_{j\ne k}f_j(X^j_{\tau^k\circ\tilde{\tau}(i)})
    + W_{\tau^k\circ\tilde{\tau}(i)}
  \end{align*}
  which, given Lemma \ref{lemma_X_tilde_distribution},
  can be seen as a mixture of i.i.d random variables.

  Lemma \ref{lemma_X_tilde_distribution}
  shows that
  there are two sources of randomness in the distribution of
  $\wtilde{Y}^k_{L,i}$:  the
  mixing variable $\Theta$ and the sub-Gaussian
  randomness of $\wtilde{X}^k_L | \Theta$ and
  $(X^{j\ne k}, W)$. For the second source,
  it is possible to use standard concentration inequalities
  to show that 
  conditioned on $\Theta$,
  $\expcthat{\wtilde{Y}^k_{L}}$
  concentrates around $\expct{\wtilde{Y}^k_{L, 1} | \Theta}$.
  We will formally do this in Lemma \ref{lemma_conentration_y_tilde}.
  Before we do this however, we focus on the first source
  and how $\Theta$ affects the distribution of
  $\wtilde{Y}^k_{L, i}$.
  
  Since $\Theta$ is sampled from $\text{Beta}(n_L + 1, n_R)$,
  we can use standard techniques to show that
  it concentrates around $\frac{1}{2}$.
  More formally, we can use the following lemma, the proof of which is in the appendix.
  \begin{lemma}\label{lemma_bound_theta}
    If $n \ge 5$, we have $\Theta \in [\frac{1}{4}, \frac{3}{4}]$
    with probability at least
    $1 - 2e^{-cn}$.
  \end{lemma}
  \noindent Given the above result,
  we can analyze $\wtilde{Y}^k_L$ assuming
  $\Theta \in [1/4, 3/4]$.
  In this case, we can use concentration inequalities to show that
  with high probability,
  $\expcthat{\wtilde{Y}^k_L}$
  concentrates around 
  $\expct{f_k(\unif(0, \Theta)})$.
  Since $f_k$ was assumed to be increasing,
  this can be further bounded by
  $\expct{f_k(\unif(0, \frac{3}{4})})$.
  Formally, we obtain the following result.
  \begin{lemma}\label{lemma_conentration_y_tilde}
    Let $k\in S$ be an active feature.
    For any $\theta\in [\frac{1}{4}, \frac{3}{4}]$,
    \begin{align*}
      \Pr \left[
      \expcthat{\wtilde{Y}^k_L} - \expct{\wtilde{Y}^k_{L, 1} | \Theta = \theta} 
        \ge t
        | \Theta = \theta
      \right]
      \le 2e^{-cn\frac{t^2}{\sigma^2}}.
    \end{align*}
    Furthermore,
    letting $\overline{f}^k_{a, b}$ denote
    $\expct{f_k(\unif(a, b))}$,
    \begin{align}
      \Pr\left[\expcthat{{Y}^k_L} \ge
      \overline{f}^k_{0, \frac{3}{4}} + g_k/8\right]
      \le 
      2e^{-c \cdot n} + 2e^{-c \cdot n \cdot \frac{g_k^2}{\sigma^2}}.
      \label{eq_jan9_1500}
    \end{align}
  \end{lemma}
  \vspace{-0.39cm}
  \begin{proof}
  For ease of notation,
    we will fix $\theta \in [\frac{1}{4}, \frac{3}{4}]$
  and let the random variables
  $\widehat{X}^k_L$ and $\widehat{Y}^k_L$
  denote
  $\wtilde{X}^k_L | \Theta=\theta$
  and 
  $\wtilde{Y}^k_L | \Theta=\theta$
  respectively.
  Recall that for all $j$, $f_k(X^j_i)$ was sub-Gaussian with parameter
  $\sigma_j$ by assumption. 
  It is straightforward to show (see 
  the Appendix) that this means
  $f_k(\widehat{X}^k_{L, i}) - \expct{f(\widehat{X}^k_{L, i})}$ is also sub-Gaussian with 
  norm at most $C \cdot \sigma_j^2$. Thus,
  \begin{align*}
    \norm{\widehat{Y}^k_{L, i}
    }_{\psi_2} 
    &=
    \norm{
      f_k(\widehat{X}^k_{L, i})
      + \sum_{j\ne k} f_j(X^j_{\tau^k\circ\tilde{\tau}(i)}) 
      + W_{\tau^k\circ\tilde{\tau}(i)}
    }_{\psi_2}^2
    \\&\overset{(i)}{=}
    \norm{
      f_k(\widehat{X}^k_{L, i})
      + \sum_{j\ne k} f_j(X^j_i) 
      + W_i
    }_{\psi_2}^2
    \\&\overset{(ii)}{\le}
    \norm{f_k(\widehat{X}^k_{L, i})}_{\psi_2}^2
    + \sum_{j\ne k} \norm{f_j(X^j_i)}_{\psi_2}^2
    + \norm{W_i}_{\psi_2}^2
    \\&\le C \cdot \sigma_{k}^2
    + \sum_{j\ne k}\sigma_j^2 + \sigma_{w}^2
    \le C \cdot \sigma^2.
  \end{align*}
  In the above analysis,
  $(i)$ follows
  from
  the independence assumption of
  $(\widehat{X}^k_{L}, X^{j\ne k}, W)$
  together with the i.i.d assumption on
  $(X^{j\ne k}_i, W_i)$.
  As for $(ii)$, it follows from \ref{P2} together with the independence assumption of $(\widehat{X}^k_{L}, X^{j\ne k}, W)$.
  Property \ref{P3} further implies that
  $ \norm{\widehat{Y}^k_{L, i} - \expct{\widehat{Y}^k_{L, i}} }_{\psi_2}^2 $
  is upper bounded by $C \cdot \sigma^2$, proving the first Equation in the Lemma.

  Now, using Hoeffding's inequality, we obtain
  \begin{align*}
    \Pr \left[
      \expcthat{\widehat{Y}^k_{L}} - \expct{\widehat{Y}^k_{L, i}} \ge g_k/8 
    \right]
    \le 2e^{-c \cdot n \cdot \frac{g_k^2}{\sigma^2}}.
  \end{align*}
  Using Lemma \ref{lemma_bound_theta} with
  $\Pr(A)\le \Pr(B) + \Pr(A|B^C)$ for
  any two events $A, B$, we obtain
  \begin{align*}
    \Pr \left[
      \expcthat{Y^k_{L}} - \expct{\widehat{Y}^k_{L, i}} \ge g_k/8 
    \right]
    \le 2e^{-c\dot n} + 2e^{-c \cdot n \cdot \frac{g_k^2}{\sigma^2}}.
  \end{align*}
  Note however that
  $\expct{\widehat{Y}^k_{L, i}} = \overline{f}^k_{0, \theta}$
  which as we show in the appendix,
  can further be upper bounded by
  $\overline{f}^k_{0, \frac{3}{4}}$,
  concluding the proof.
  \end{proof}
  \noindent Using the symmetry of the decision tree algorithm,
  we can further obtain that
  \begin{align}
    \Pr\left[\expcthat{{Y}^k_R} \ge
    \overline{f}^k_{\frac{1}{4}, 1} - g_k/8\right]
    \le 
    2e^{-c \cdot n} + 2e^{-c \cdot n \cdot \frac{g_k^2}{\sigma^2}}
    \label{eq_jan9_1516}
  \end{align}
  from
  \eqref{eq_jan9_1500}
  with the change of variable $X^k \to -X^k$ and
  $f_k=-f_k$.
  Taking union bound over
  \eqref{eq_jan9_1500} and \eqref{eq_jan9_1516},
  it follows that with probability at least
  $1 - 4e^{-c \cdot n} - 4e^{-c \cdot n \cdot \frac{g_k^2}{\sigma^2}}$,
  \begin{align*}
      \expcthat{Y^k_{R}} - \expcthat{Y^k_{L}}
      \ge 
      \overline{f}^k_{\frac{1}{4}, 1} - \overline{f}^k_{0, \frac{3}{4}}
      - g_k/4.
  \end{align*}
  As we show 
  in the appendix however,
  a simple application of
  conditional expectations implies
  $\overline{f}^k_{\frac{1}{4}, 1} - \overline{f}^k_{0, \frac{3}{4}} \ge g_k/3$.
  Therefore,
  with probability at least
  $1 - 4e^{-c \cdot n} - 4e^{-c \cdot n \cdot \frac{g_k^2}{\sigma^2}}$, we have
  $
      \expcthat{Y^k_{R}} - \expcthat{Y^k_{L}}
      \ge \frac{g_k}{12}.
  $
  Assuming 
  $n\ge 5$,
  we can further conclude that
  $\frac{n_L \cdot n_R}{n^2}\ge \frac{1}{5}$
  which together with \eqref{mdi_identity}, proves the lemma.

\section{Experimental Results} \label{exp_results}
\begin{figure*}[ht]
    \center{\includegraphics[width=\textwidth]{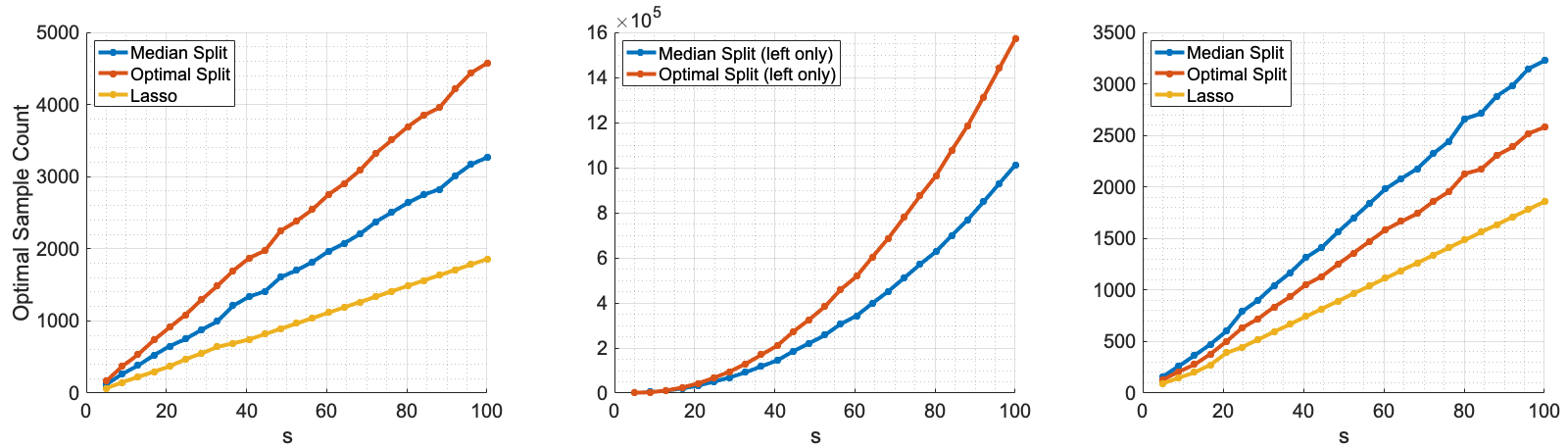}}
    \caption{Optimal sample count to recover 95\% of the active features where design matrix samples i.i.d from $U(-1,1)$ or $\mathcal{N}(0,1)$ with additive Gaussian noise $\mathcal{N}(0,0.1)$, comparing three methods: \dstump with optimal split, \dstump with median split, and \lasso.}
    \label{fig:exp_results}
\end{figure*}

In this section, we provide further justification of our theoretical results in the form of simulations on the finite sample count for active feature recovery under different regimes, as well as the predictive power of a single sub-tree as compared to the full tree. We additionally contrast \dstump with the widely studied optimal \lasso.

We first validate the result of Theorem \ref{thm_linear} and consider the
linear design with $p=200$ and design matrix entries sampled i.i.d. from
$U(0,1)$ with additive Gaussian noise $\mathcal{N}(0,.1)$. Concretely, we
examine the optimal number of samples required to recover approximately $95\%$
of the active features $s$. This is achieved by conducting a binary search on
the number of samples to find the minimal such value that recovers the desired
fraction of the active feature set, averaged across 25 independent
replications. In the leftmost plot of Figure \ref{fig:exp_results}, we plot the
sample count as a function of varying sparsity levels $s \in [5,100]$ for
\dstump with a median split, \dstump with the optimal split, as well as \lasso
for bench-marking (with penalty parameter selected through standard
cross-validation). By fixing $p$, we are evaluating the dependence of $n$ on
the sparsity level $s$ alone. The results here validate the theoretical $\mO(s
\log p)$ bound that nearly matches the optimal \lasso. Also of note, the number
of samples required by the median splitting is less than that of the optimal.
Thus, in the linear setting, we see that \dstump with median splitting is both
more simplistic and computationally inexpensive. This optimal bound result is
repeated with
Gaussian data samples in the right most plot of Figure \ref{fig:exp_results}.
Notably, in this setting the optimal split decision stumps perform better than
the median as it demonstrates their varied utility under different problem
contexts.

We additionally reiterate that the prior work of
\cite{Kazemitabar2017} attempted to simplify the analysis of
the sparse recovery problem using \dstump by examining only the left sub-tree,
which produced the non-optimal $\mO(s^2 \log p)$ finite sample bound.
To analyze the effect of this choice, the middle plot of Figure
\ref{fig:exp_results} presents the optimal sample recovery count when using
only the left sub-tree subject to the additive model of Theorem
\ref{thm_linear}. In accordance with our expectation and previous literature's
analysis, we see a clear quadratic relationship between $n$ and $s$ when fixing
the feature count $p$.

Overall, these simulations further validate the practicality and predictive
power of decisions stumps. Benchmarked against the optimal \lasso, we see a
slight decrease in performance but a computational reduction
and analytic simplification.

\section{Conclusion} \label{conclusion}
In this paper, we presented a simple and consistent feature selection algorithm
in the regression case with single-depth decision trees, and derived the
finite-sample performance guarantees in a high-dimensional sparse system. Our
theoretical results demonstrate that this very simple class of weak learners is
nearly optimal compared to the gold standard \lasso.
We
have provided strong theoretical evidence for the success of binary decision
tree based methods in practice and provided a framework on which to extend the
analysis of these structures to arbitrary height, a potential direction for
future work.

\section{Acknowledgements}
The work is partially support by DARPA QuICC, NSF AF:Small \#2218678, and  NSF AF:Small \# 2114269. 
Max Springer was supported by the National Science Foundation Graduate Research Fellowship Program under Grant No. DGE 1840340.
Any opinions, findings, and conclusions or recommendations expressed in this
material are those of the author(s) and do not necessarily reflect the views of
the National Science Foundation.

\bibliography{main}

\begin{thebibliography}{27}
\providecommand{\natexlab}[1]{#1}

\bibitem[{Barber and Cand{\`e}s(2015)}]{barber2015controlling}
Barber, R.~F.; and Cand{\`e}s, E.~J. 2015.
\newblock Controlling the false discovery rate via knockoffs.
\newblock \emph{The Annals of Statistics}, 43(5): 2055--2085.

\bibitem[{Breiman(2001)}]{Breiman_2001}
Breiman, L. 2001.
\newblock {Statistical Modeling: The Two Cultures (with comments and a
  rejoinder by the author)}.
\newblock \emph{Statistical Science}, 16(3): 199 -- 231.

\bibitem[{Breiman et~al.(1983)Breiman, Friedman, Olshen, and
  Stone}]{Breiman1983}
Breiman, L.; Friedman, J.~H.; Olshen, R.~A.; and Stone, C.~J. 1983.
\newblock Classification and Regression Trees.

\bibitem[{Bureau et~al.(2005)Bureau, Dupuis, Falls, Lunetta, Hayward, Keith,
  and Eerdewegh}]{Bureau_2005}
Bureau, A.; Dupuis, J.; Falls, K.; Lunetta, K.; Hayward, B.; Keith, T.; and
  Eerdewegh, P. 2005.
\newblock Identifying SNPs predictive of phenotype using random forests.
\newblock \emph{Genetic epidemiology}, 28: 171--82.

\bibitem[{Chung and Romano(2013)}]{chung2013exact}
Chung, E.; and Romano, J.~P. 2013.
\newblock Exact and asymptotically robust permutation tests.
\newblock \emph{The Annals of Statistics}, 41(2): 484--507.

\bibitem[{Chung and Romano(2016)}]{chung2016multivariate}
Chung, E.; and Romano, J.~P. 2016.
\newblock Multivariate and multiple permutation tests.
\newblock \emph{Journal of econometrics}, 193(1): 76--91.

\bibitem[{{Duroux, Roxane} and {Scornet, Erwan}(2018)}]{Duroux_2018}
{Duroux, Roxane}; and {Scornet, Erwan}. 2018.
\newblock Impact of subsampling and tree depth on random forests.
\newblock \emph{ESAIM: PS}, 22: 96--128.

\bibitem[{Fan, Feng, and Song(2011)}]{fan2011nonparametric}
Fan, J.; Feng, Y.; and Song, R. 2011.
\newblock Nonparametric independence screening in sparse ultra-high-dimensional
  additive models.
\newblock \emph{Journal of the American Statistical Association}, 106(494):
  544--557.

\bibitem[{Fan and Lv(2006)}]{Fan2006}
Fan, J.; and Lv, J. 2006.
\newblock Sure independence screening for ultrahigh dimensional feature space.
\newblock \emph{Journal of The Royal Statistical Society Series B-statistical
  Methodology}, 70: 849--911.

\bibitem[{Friedman(2001)}]{Friedman2001}
Friedman, J.~H. 2001.
\newblock Greedy function approximation: A gradient boosting machine.
\newblock \emph{Annals of Statistics}, 29: 1189--1232.

\bibitem[{Han et~al.(2020)Han, Rao, Sorokina, and Subbian}]{Han2020}
Han, C.; Rao, N.~S.; Sorokina, D.; and Subbian, K. 2020.
\newblock Scalable Feature Selection for (Multitask) Gradient Boosted Trees.
\newblock \emph{ArXiv}, abs/2109.01965.

\bibitem[{Huynh-Thu et~al.(2012)Huynh-Thu, Irrthum, Wehenkel, and
  Geurts}]{HuynhThu_2010}
Huynh-Thu, V.~A.; Irrthum, A.; Wehenkel, L.; and Geurts, P. 2012.
\newblock Inferring Regulatory Networks from Expression Data Using Tree-Based
  Methods.
\newblock \emph{PLOS ONE}, 5(9): 1--10.

\bibitem[{Kazemitabar et~al.(2017)Kazemitabar, Amini, Bloniarz, and
  Talwalkar}]{Kazemitabar2017}
Kazemitabar, S.~J.; Amini, A.~A.; Bloniarz, A.; and Talwalkar, A.~S. 2017.
\newblock Variable Importance Using Decision Trees.
\newblock In \emph{NIPS}.

\bibitem[{Klusowski(2021)}]{Klusowski_21b}
Klusowski, J. 2021.
\newblock Sharp Analysis of a Simple Model for Random Forests.
\newblock In Banerjee, A.; and Fukumizu, K., eds., \emph{Proceedings of The
  24th International Conference on Artificial Intelligence and Statistics},
  volume 130 of \emph{Proceedings of Machine Learning Research}, 757--765.
  PMLR.

\bibitem[{Klusowski(2020)}]{Klusowski2020}
Klusowski, J.~M. 2020.
\newblock Sparse learning with CART.
\newblock \emph{ArXiv}, abs/2006.04266.

\bibitem[{Klusowski and Tian(2021)}]{Klusowski2021}
Klusowski, J.~M.; and Tian, P.~M. 2021.
\newblock Nonparametric Variable Screening with Optimal Decision Stumps.
\newblock In \emph{AISTATS}.

\bibitem[{Lafferty and Wasserman(2008)}]{Lafferty2008}
Lafferty, J.~D.; and Wasserman, L.~A. 2008.
\newblock Rodeo: Sparse, greedy nonparametric regression.
\newblock \emph{Annals of Statistics}, 36: 28--63.

\bibitem[{Li and Wang(2018)}]{Li_2018}
Li, L.; and Wang, J. 2018.
\newblock Research on feature importance evaluation of wireless signal
  recognition based on decision tree algorithm in cognitive computing.
\newblock \emph{Cognitive Systems Research}, 52: 882--890.

\bibitem[{Li et~al.(2019)Li, Wang, Basu, Kumbier, and Yu}]{Li2019}
Li, X.; Wang, Y.; Basu, S.; Kumbier, K.; and Yu, B. 2019.
\newblock A Debiased mDI Feature Importance Measure for Random Forests.
\newblock \emph{ArXiv}, abs/1906.10845.

\bibitem[{Lunetta et~al.(2004)Lunetta, Hayward, Segal, and
  Eerdewegh}]{Lunetta_2004}
Lunetta, K.~L.; Hayward, L.~B.; Segal, J.; and Eerdewegh, P.~V. 2004.
\newblock Screening large-scale association study data: exploiting interactions
  using random forests.
\newblock \emph{{BMC} Genetics}, 5(1): 32.

\bibitem[{Qian et~al.(2022)Qian, Wang, Yuan, Gao, and Song}]{Qian_2022}
Qian, H.; Wang, B.; Yuan, M.; Gao, S.; and Song, Y. 2022.
\newblock Financial distress prediction using a corrected feature selection
  measure and gradient boosted decision tree.
\newblock \emph{Expert Systems with Applications}, 190: 116202.

\bibitem[{Rundel(2012)}]{dukeOrderStat}
Rundel, C. 2012.
\newblock Lecture 15: Order statistics.
\newblock
  \url{https://www2.stat.duke.edu/courses/Spring12/sta104.1/Lectures/Lec15.pdf}.

\bibitem[{Vershynin(2018)}]{Vershynin2018}
Vershynin, R. 2018.
\newblock High-Dimensional Probability: An Introduction with Applications in
  Data Science.

\bibitem[{Wainwright(2009{\natexlab{a}})}]{Wainwright2009_1}
Wainwright, M.~J. 2009{\natexlab{a}}.
\newblock Information-Theoretic Limits on Sparsity Recovery in the
  High-Dimensional and Noisy Setting.
\newblock \emph{IEEE Transactions on Information Theory}, 55: 5728--5741.

\bibitem[{Wainwright(2009{\natexlab{b}})}]{Wainwright2009_2}
Wainwright, M.~J. 2009{\natexlab{b}}.
\newblock Sharp thresholds for high-dimensional and noisy sparsity recovery
  using l1-constrained quadratic programming (Lasso).
\newblock \emph{IEEE Transactions on Information Theory}.

\bibitem[{Xia et~al.(2017)Xia, Liu, Li, and Liu}]{Xia_2017}
Xia, Y.; Liu, C.; Li, Y.; and Liu, N. 2017.
\newblock A boosted decision tree approach using Bayesian hyper-parameter
  optimization for credit scoring.
\newblock \emph{Expert Systems with Applications}, 78: 225--241.

\bibitem[{Xu et~al.(2014)Xu, Huang, Weinberger, and Zheng}]{xu2014gradient}
Xu, Z.; Huang, G.; Weinberger, K.~Q.; and Zheng, A.~X. 2014.
\newblock Gradient boosted feature selection.
\newblock In \emph{Proceedings of the 20th ACM SIGKDD international conference
  on Knowledge discovery and data mining}, 522--531.

\end{thebibliography}

\ifappendix
\appendix 
\onecolumn
\newpage
\section{Appendix}

\subsection{Extension to non-uniform distributions}
In this section, we extend our main result, i.e., 
Theorem \ref{thm_additive_monotonic},
to non-uniform distributions.
\begin{theorem}\label{thm_general_dist}
  Assume that the entries of the input matrix
  $X$ are sampled independently, the entries in each column are i.i.d. 
  and $Y = \sum_k f_k(X^k) + W$
  where $W_i$ are
  sampled i.i.d from $\mathcal{N}(0, \sigma_{w}^2)$.
  Assume further that
  each $f_k$ is monotone and 
  $f_k(\bX_{i, k})$ is sub-Gaussian with sub-Gaussian norm
  $\norm{f_k(\bX_{i, k})}_{\psi_2}^2 \le \sigma_{k}^2$.
  ~\\
  Let $F_k$ denote the CDF of the distribution of $X^k_i$.
  Define $F_k^{-1, +}(z) := \inf \{t: F_k(t) > z \}$
  and define $h_k := f_k \circ F_k^{-1, +}$.
  For $k\in S$,
  define $g_k$ as 
  \begin{align}
    g_k := \expct{h_k(\unif(\frac{1}{2}, 1))} - \expct{h_k(\unif(0, \frac{1}{2}))},
    \label{eq:g_k_2}
  \end{align}
  and define
  $\sigma^2$ as $\sigma^2 = \sigma_w^2 + \sum_k \sigma_k^2$.
  Algorithm \ref{alg:stump} correctly recovers the set $S$ with probability at least
$$
    1 - 4se^{-cn} - 4pe^{-cn\frac{\min_{k} g_k^2}{\sigma^2}}
$$
  for the median split and
$$
    1 - 4se^{-cn} - 4npe^{-cn\frac{\min_{k} g_k^2}{\sigma^2}}
$$
  for the optimal split.
\end{theorem}
  \begin{proof}
  We first claim that if $Z\sim \unif(0, 1)$, then $F_{k}^{-1, +}(Z)$
  has the same distribution as $X^k_{i}$. Formally,
  for all $z\in [0, 1]$,
  \begin{align*}
    F_{k}^{-1, +}(z) \le t
    \iff
    \inf \{\tilde{t}: F_k(\tilde{t}) > z \} \le t
    &\overset{(i)}{\iff}
    \forall \delta > 0: \exists \tilde{t} \le t + \delta: F_{k}(\tilde{t}) > z
    \\&\overset{(ii)}{\iff}
    \lim_{\epsilon \to 0, \epsilon > 0} F_k(t + \epsilon) > z
    \\&\overset{(iii)}{\iff}
    F_k(t) > z.
  \end{align*}
  where $(i)$ follows from the definition of $\inf$,
  $(ii)$ follows from the fact that $F_k$ is increasing and $(iii)$
  follows from the fact that $F_k$ is right continuous.
  Therefore,
  \begin{align*}
    \Pr \left(F_k^{-1, +}(Z) \le t\right)
    &=
    \Pr \left(Z < F_k(t)\right) \overset{(i)}=F_k(t).
  \end{align*}
  where for $(i)$ we have used the fact that $Z$
  is sampled uniformly on $[0, 1]$. 
  Since $F_k(t) = \Pr(X^k_i \le t)$, the claim is proved.

  Given this result, we can assume that $X^k_{i}$ were
  set as
  $F_k^{-1, +}(Z^k_i)$ for a latent variable
  $Z^k_i \sim \unif(0, 1)$. While this means $Z^k_i$ is no longer a \emph{function} of $X^k_i$,
  $\text{argsort}(Z^k)$
  still has the same distribution as
  $\text{argsort}(X^k)$.
  To see why, consider a fixed value of $X^k$
  and sample $Z^k$ based on the conditional distribution.
  If $X^k_{i} < X^k_{j}$ for some
  $i, j$, then $Z^k_{i} < Z^k_j$ as well
  since $F^k_{-1, +}$ is a function. As for the case
  case of $X^k_{i} = X^k_{j}$, then
  $Z^k_{i} < Z^k_{j}$ and $Z^k_{i} > Z^k_j$ are equi-probable which is consistent with the assumption that $\text{argsort}$ breaks ties randomly.
  ~\\
  We now note that
  \begin{align*}
    f_k(X^k_i) = f_k(F_k^{-1, +}(Z^k_i)) = h_k(Z^k_i)
    .
  \end{align*}
  Since $\text{argsort}(Z^k)$ and $\text{argsort}(X^k)$ have the same distribution,
  Algorithm \ref{alg:stump}
  has the same output on $(\bX, f_k)$ as $(\bZ, h_k)$.
  Invoking Theorem \ref{thm_additive_monotonic} therefore completes the proof.

  \end{proof}
\subsection{Auxiliary lemmas}
\begin{lemma}\label{lemma_sub_Gaussian_f_delta}
  Let $U \sim \unif(0, 1)$ be uniform random variable on $[0, 1]$
  and
  Let $f$ be a
  function such that $f(U)$ is sub-Gaussian with
  norm $\norm{f(U)}_{\psi_2}$.
  Let $\delta$ be a value in the range $[\frac{1}{4}, \frac{3}{4}]$,
  and let $\tilde{U} \sim \unif(0, \delta)$ be a random variable
  on $[0, \delta]$.
  Then $f(U)$ is sub-Gaussian as well
  with norm at most a constant times that of
  $f(U)$, i.e,
  \begin{align*}
    \norm{f(\tilde{U})}_{\psi_2} \le C \cdot \norm{f(U)}_{\psi_2} 
  \end{align*}
\end{lemma}
\begin{proof}
  Define $K:=\norm{f(U)}_{\psi_2}$,
  \begin{align}
    \notag
    2 &\overset{(i)}{\ge} 
    \expct{e^{\frac{f(U)^2}{K^2}}}
    \\&=
    \notag
    \expct{e^{\frac{f(U)^2}{K^2}} \Big| U \le \delta}
    \cdot \Pr \left[ U \le \delta \right]
    +
    \expct{e^{\frac{f(U)^2}{K^2}} \Big| U > \delta}
    \cdot \Pr \left[ U > \delta \right]
    \\&\overset{(ii)}{\ge}
    \notag
    \expct{e^{\frac{f(U)^2}{K^2}} \Big| U \le \delta}
    \cdot \Pr \left[ U \le \delta \right]
    \\&=
    \expct{e^{\frac{f(\tilde{U})^2}{K^2}}}
    \cdot
    \delta
    \label{eq_jan1_1640}
  \end{align}
  where $(i)$ follows from the definition of the sub-Gaussian norm
  and $(ii)$ follows from the fact that $e^{x} \ge 0$ for
  all $x$.
  Define $\wtilde{K} := K \cdot \frac{ln\left(2/\delta\right)}{ln(2)}$.
  Note that $\delta < 1$ and therefore $\frac{K}{\wtilde{K}} < 1$.
  We can therefore obtain.
  \begin{align*}
    \expct{e^{\frac{f(\tilde{U})^2}{\wtilde{K}^2}}}
    &=
    \expct{e^{\frac{f(\tilde{U})^2}{K^2} \cdot \frac{K^2}{\wtilde{K}^2}}}
    \\&=
    \expct{\left(e^{\frac{f(\tilde{U})^2}{K^2}}\right)^{\frac{K^2}{\wtilde{K}^2}}}
    \\&\overset{(i)}{\le}
    \expct{\left(e^{\frac{f(\tilde{U})^2}{K^2}}\right)}^{\frac{K^2}{\wtilde{K}^2}}
    \\&\overset{(ii)}{\le}
    \left(2/\delta\right)^{\frac{K^2}{\wtilde{K}^2}}
    \\&\overset{(iii)}{=}
    e^{\ln(2/\delta) \cdot \frac{ln(2)}{ln(2/\delta)}}
    =2
  \end{align*}
  where for $(i)$ we have used the Jensen's inequality for
  the concave function $t\to t^{\frac{K^2}{\wtilde{K}^2}}$,
  for $(ii)$ we have used \eqref{eq_jan1_1640}
  and for $(iii)$ we have used the definition of $\wtilde{K}$.
  Therefore the claim is proved with $C=\frac{ln\left(2/\delta\right)}{ln(2)}$.
\end{proof}
\eat{
\begin{lemma}
  For any $0 \le a \le c \le b \le 1$,
  \begin{align*}
    \overline{f}^k_{a, b} = 
    \frac{c - a}{b - a} \cdot \overline{f}^k_{a, c} 
    +
    \frac{b - c}{b - a} \cdot \overline{f}^k_{c, b} 
  \end{align*}
\end{lemma}
\begin{proof}
  Let $Z$ be a random variable distributed as $\unif(a, b)$.
  By law of total expectation,
  \begin{align*}
    \overline{f}^k_{a, b} &= 
    \expct{f_k(Z)}
    \\&=
    \Pr\left({Z < c}\right) \cdot \expct{f_k(Z) | Z < c} 
    + 
    \Pr\left({Z \ge c}\right) \cdot \expct{f_k(Z) | Z \ge c} 
    \\&=
    \frac{c - a}{b - a} \cdot \expct{f_k(\unif(a, c))}
    +
    \frac{b - c}{b - a} \cdot \expct{f_k(\unif(c, 1))}
    \\&=
    \frac{c - a}{b - a} \cdot \overline{f}^k_{a, c}
    +
    \frac{b - c}{b - a} \cdot \overline{f}^k_{c, b}
  \end{align*}
\end{proof}
}
\begin{lemma}\label{f_k_increasing}
  If $f_k$ is increasing, then
  for any $0 \le a \le c \le b \le 1$,
  $\overline{f}^k_{a, c} \le \overline{f}^k_{c, b}$.
\end{lemma}
\begin{proof}
  \begin{align*}
    \overline{f}^k_{a, c}
    =
    \expctu{Z \sim \unif(a, c)}{f_k(Z)}
    \le
    f_k(c)
    \le
    \expctu{Z \sim \unif(c, b)}{c}
    =
    \overline{f}^k_{c, b}
  \end{align*} 
\end{proof}
\begin{lemma}
  If $f_k$ is increasing, then
  for any $0 \le a \le c \le b \le 1$,
  $\overline{f}^k_{a, c} \le \overline{f}^k_{a, b}$.
\end{lemma}
\begin{proof}
  \newcommand{\prob}[1]{\Pr\left[ #1 \right]}
  \begin{align*}
    \overline{f}^k_{a, b}
    =
    \expctu{Z \sim \unif(a, b)}{f_k(Z)}
    &=
    \prob{Z \in [a, c]} \cdot \expctu{Z \sim \unif(a, c)}{f_k(Z)}
    + 
    \prob{Z \in [c, b]} \cdot \expctu{Z \sim \unif(c, b)}{f_k(Z)}
    \\&=
    \prob{Z \in [a, c]} \cdot 
    \overline{f}^k_{a, c}
    + 
    \prob{Z \in [c, b]} \cdot
    \overline{f}^k_{c, b}
    \\&\ge
    \overline{f}^k_{a, c}
  \end{align*}
\end{proof}

\begin{lemma}\label{lemma_f_1_4_3_4_1_2}
  \begin{align*}
      \overline{f}^k_{\frac{1}{4}, 1} - \overline{f}^k_{0, \frac{3}{4}} \ge g_k/3
  \end{align*}
\end{lemma}
\begin{proof}
  We have
  \begin{align*}
    \overline{f}^k_{0, \frac{3}{4}} =
    \frac{2}{3} \cdot \overline{f}^k_{\frac{1}{4}, \frac{3}{4}}
    +
    \frac{1}{3} \cdot \overline{f}^k_{0, \frac{1}{4}}
    ,
  \end{align*}
  and
  \begin{align*}
    \overline{f}^k_{\frac{1}{4}, 1} =
    \frac{2}{3} \cdot \overline{f}^k_{\frac{1}{4}, \frac{3}{4}}
    +
    \frac{1}{3} \cdot \overline{f}^k_{\frac{3}{4}, 1}
    .
  \end{align*}
  It therefore follows that
  \begin{align*}
    \overline{f}^k_{\frac{1}{4}, 1} - \overline{f}^k_{0, \frac{3}{4}} = \frac{1}{3} \cdot
    \left(
    \overline{f}^k_{\frac{3}{4}, 1}
    -
    \overline{f}^k_{0, \frac{1}{4}}
    \right)
  \end{align*}
  Note however that 
  \begin{align*}
    g_k = 
    \overline{f}^k_{\frac{1}{2}, 1}
    - 
    \overline{f}^k_{0, \frac{1}{2}}
    &=
    \left(
    \frac{1}{2} \cdot \overline{f}^k_{\frac{1}{2}, \frac{3}{4}}
    +
    \frac{1}{2} \cdot \overline{f}^k_{\frac{3}{4}, 1}
    \right)
    -
    \left(
    \frac{1}{2} \cdot \overline{f}^k_{0, \frac{1}{4}}
    +
    \frac{1}{2} \cdot \overline{f}^k_{\frac{1}{4}, \frac{1}{2}}
    \right)
    \\&=
    \left(
    \frac{1}{2} \cdot \overline{f}^k_{\frac{3}{4}, 1}
    -
    \frac{1}{2} \cdot \overline{f}^k_{0, \frac{1}{4}}
    \right)
    +
    \left(
    \frac{1}{2} \cdot \overline{f}^k_{\frac{1}{2}, \frac{3}{4}}
    -
    \frac{1}{2} \cdot \overline{f}^k_{\frac{1}{4}, \frac{1}{2}}
    \right)
    \\&\le
    \overline{f}^k_{\frac{3}{4}, 1}
    - 
    \overline{f}^k_{0, \frac{1}{4}}
  \end{align*}
  which proves the claim.
\end{proof}
  
  \subsection{Proof of Corollary \ref{cor.only}}
  \label{sec.appendix.cor}
  \begin{proof}
    We need to show that 
    for $n \gtrsim s \log(p)$,
    \begin{align*}
        &\lim_{p \to \infty}
        1 
         - 4s e^{-c\cdot n}
        - 4npe^{-c \cdot n\frac{\min_{k}\beta_k^2}{\norm{\beta}_2^2 + \sigma_{w}^2}}
        = 1
    \end{align*}
    First, note that
    since $\min_{k} \beta_k^2 \in \Omega(\frac{1}{s})$
    and $\norm{\beta}_2 \in \mO(1)$,
    it follows that
    $\frac{\min_{k}\beta_k^2}{\norm{\beta}_2^2 + \sigma_w^2} \in \Omega(\frac{1}{s})$.
    Assuming that
    $\frac{\min_{k}\beta_k^2}{\norm{\beta}_2^2 + \sigma_w^2} \ge \frac{C}{s}$ for some $C$,
    set 
    $n = \frac{\max\{4/C, 2\}}{c}s\log(p)$.
    This implies that
    \begin{align*}
        4se^{-cn} \le 4\frac{s}{p^{2s}} \le \frac{4}{p}.
    \end{align*}
    which goes to zero for large $p$.
    It therefore remains to show that
    $4npe^{-cn \frac{\min_{k}\beta_k^2}{\norm{\beta}_2^2 + \sigma_w^2}}$ goes to zero
    which is equivalent to showing that
    $cn \frac{\min_{k}\beta_k^2}{\norm{\beta}_2^2 + \sigma_w^2} - \log(np)$ goes to $\infty$.
    Note however that
    \begin{align*}
        cn \frac{\min_{k}\beta_k^2}{\norm{\beta}_2^2 + \sigma_w^2} - \log(np)
        &= 
        cn \frac{\min_{k}\beta_k^2}{\norm{\beta}_2^2 + \sigma_w^2} - \log(n) - \log(p)
        \\&\ge
        c \cdot  \frac{4}{c \cdot C} s\log(p) \cdot \frac{C}{s}
        - \log(p)
        - \log(n)
        \\&=
        3\log(p) - \log(n)
    \end{align*}
    Now, observe that by our choice of $n$,
    \begin{align*}
      n \le (\frac{4}{c\cdot C} + \frac{2}{c})\cdot p\log(p) \le (\frac{4}{c\cdot C} + \frac{2}{c})p^2
    .
    \end{align*}
    Therefore,
    $\log(n) \le 2\log(p) + \log(C')$, where
    $C' := \frac{4}{c\cdot C} + \frac{2}{c}$, implying that
    \begin{align*}
       3\log(p) - \log(n) \ge 
       \log(p) - \log(C'),
    \end{align*}
    which goes to $\infty$ as large $p$.
  \end{proof}
  
  \subsection{Proof of Lemma \ref{lemma_important_is_S}}
  In this section, we provide the proof of Lemma \ref{lemma_important_is_S}, which is a generalization of Lemma 1 in \cite{Li2019}.
\begin{proof}
    Since $k'$ is independent of $Y$,
  so is the permutation $\tau^{k'}$.
  Since $Y_i$ were assumed to be i.i.d,
  this implies that $Y_{\tau^k(i)}$, and by extension $Y^k_{L, i}$,
  are i.i.d as well and have the same distribution as
  $Y_i$.
  $Y^k_i$ are zero-mean and sub-Gaussian with norm at most
  $\sigma$ however as
  \begin{align*}
    \norm{Y_i}_{\psi_2}^2
    =
    \norm{f_k(X^k_i) + \sum_{j\ne k} f_j(X^j_i) + W}_{\psi_2}^2
    &\overset{(i)}\le
    \norm{f_k(X^k_i)}_{\psi_2} + \sum_{j\ne k} \norm{f_j(X^j_i)}_{\psi_2} + \norm{W}_{\psi_2}^2
    \\&= \sigma^2.
  \end{align*}
  where $(i)$ follows from \hyperref[]{\ref{P2}}.
  Focusing on the median split,
  Hoeffding's inequality therefore implies that
  for any $t > 0$,
  $\Pr\left(\left|\expcthat{Y^k_{L}}\right| > t\right) \le 2e^{-c \cdot n_{L}t^2/\sigma^2}$
  and 
  $\Pr\left(\left|\expcthat{Y^k_{R}}\right| > t\right) \le 2e^{-c \cdot n_{R}t^2/\sigma^2}$.
  Setting $t=\frac{g_k}{30}$
  and
  a union bound for both sub-trees
  implies that
  with probability
  at least
  $1 - 4e^{-c \cdot n \cdot \frac{\min_{k} g_k^2}{\sigma^2}}$,
  \begin{align*}
      \left|
      \expcthat{Y^k_{R}} - \expcthat{Y^k_{L}}
      \right|
      \le \frac{g_k}{15}.
  \end{align*}
  Since $\frac{n_{L} \cdot n_{R}}{n^2} \le \frac{1}{4}$, the claim follows from 
  \eqref{mdi_identity}.
  
  As for the optimal split,
  the analysis needs to change as the split point is not necessarily
  in the middle and is also dependent on the data.
  For a fixed $n_L$ however, 
  the same analysis as
  above can be used for bounding $\dimp_{k, n_L}$
  with small tweaking as shown by \cite{Li2019};
  while the bound on the concentration of
  $\expcthat{Y}^k_{L}$
  is weaker since $\frac{n_{L}}{n}$ can be small,
  this is offset by the $\frac{n_L \cdot n_{R}}{n^2}$ term in
  \eqref{mdi_identity} and ultimately the same bound
  on $\dimp_{k, n_L}$ can be obtained.
  Taking a union bound over all $n$ possible splitting
  points proves the results. 
  \end{proof}
  
\subsection{Proof of Lemma \ref{lemma_bound_theta}}
\begin{lemma}
  Let the random variable $\Theta$ be distributed as $\text{Beta}(n_L + 1, n_R)$ for
  $\{n_L, n_R\}=\{ \floor{\frac{n}{2}} , \ceil{\frac{n}{2}} \}$.
  For $n \ge 5$ and $t \ge \frac{2}{3}$,
  \begin{align*}
    \Pr \left(\Theta \ge t \right) \le e^{-n(t-\frac{2}{3})^2}.
  \end{align*}
\end{lemma}
\begin{proof}
  Let $U_1 , \dots , U_n$ be i.i.d $\unif(0, 1)$ random variables
  and let $\tau$ be their sorting permutation, i.e,
  \begin{align*}
    U_{\tau(1)} \le \dots \le U_{\tau(n)}
  \end{align*}
  It is well-known (e.g., see \cite{dukeOrderStat}) that for any $k$, $U_{\tau(k)}$
  is distributed as
  $\text{Beta}(k, n + 1 - k)$.
  Therefore, $\Theta$ has the same distribution as $U_{\tau(n_L + 1)}$.
  This means that for any $t\in [0, 1]$,
  \begin{align*}
    \Pr\left[
      \Theta \ge t
    \right]
    &=
    \Pr\left[
      U_{\tau(n_L + 1)} \ge t
    \right]
    \\&=
    \Pr\left[
      \left| 
        \{
          k: U_k < t
        \}
      \right|
      \le
      n_{L} 
    \right]
    \\&=
    \Pr\left[
      \left| 
        \{
          k: U_k \ge t
        \}
      \right|
      \ge
      n_R
    \right]
    \\&=
    \Pr\left[
      \sum_{k} \ind{U_k \ge t} \ge n_R
    \right]
    \\&=
    \Pr\left[
      \frac{1}{n}\cdot \left( \sum_{k} \ind{U_k \ge t} \right) \ge \frac{n_R}{n}
    \right]
  \end{align*}

  Defining $Z_k := \ind{U_k \ge t}$, it is clear
  that $Z_k$ are i.i.d Bernoulli random variables with parameter $1-t$.
  If $1-t \le \frac{n_{R}}{n}$,
  Hoeffding's inequality implies that
  \begin{align*}
    \Pr\left[
      \Theta \ge t
    \right]
    &=
    \Pr\left[
      \frac{1}{n}\cdot \left( \sum_{k} Z_k \right) \ge \frac{1}{3}
    \right]
    \\&\le
    2e^{-n\left(\frac{n_{R}}{n} - (1-t)\right)^2}
    \\&=2e^{-n\cdot (t - \frac{2}{3})^2}
  \end{align*}
  which completes the proof.
\end{proof}
By symmetry, we also have
$\Pr{\Theta \le t} \le e^{-n(t-\frac{1}{3})^2}$ for
any $t \le \frac{1}{3}$, which implies Lemma \ref{lemma_bound_theta} via a union bound.
\subsection{Proof of Theorem \ref{thm.unknown_s}}
\begin{proof}
  Consider the distribution of $\dimp_k$ for an inactive $k$.
  Since $k$ is inactive, the sorting permutation $\tau_k$ is independent of $\bX, Y$.
  Therefore, the distribution of $\dimp_k$ is the distribution of
  $\dimp(\tau(Y))$ where $\tau$ is a random permutation on $[n]$ and
  $\dimp(Y)$ is defined as
  in Lines \ref{line:imp_1}, \ref{line:imp_2}, \ref{line:imp_3} i.e, 
  \begin{align*}
  \dimp(Y) = 
    \frac{\floor{\frac{n}{2}}}{n}\varHat(Y_{\le \floor{\frac{n}{2}}}) + \frac{n - \floor{\frac{n}{2}}}{n}\varHat(Y_{> \floor{\frac{n}{2}}})
  \end{align*}
  for the median split and
  \begin{align*}
  \dimp(Y) = 
  \min_{n_{L}}
    \frac{n_L}{n}\varHat(Y_{\le n_{L}}) + \frac{n - n_{L}}{n}\varHat(Y_{> n_{L}})
  \end{align*}
  for the optimal split.
  In the above equations,
  $Y_{\le i}$ and $Y_{> i}$ refer to
  $(Y_{1}, \dots Y_{i})^T$ and
  $(Y_{i+1}, \dots, Y_n)^T$ respectively.
  
  Similarly, for $t\in [T]$ and $k\in [p]$,
  $\tau^{(t)}_i$ is a random permutation independent of $\bX, Y$.
  Therefore,
  $\dimp_{i}^{(t)}$ is also distributed as
  $\dimp(\tau(Y))$ for a random permutation $\tau$ on $[n]$. 
  Furthermore, all of these random permutations (and therefore all values $\dimp_{i}^{(t)}$ and $\dimp_k$ for inactive $k$) are independent of each other.
  This is because
  \textbf{(a)} all of the inactive features are independent of each other by assumption
  and \textbf{(b)} the permutations $\sigma^{(t)}$ were artificially independently by design as they were sampled independently in Line \ref{line:random_permute}.
  We now observe that
  $\min_{k \notin S} \dimp_k \le
  \min_{i, t} \dimp_{i}^{(t)}$
  happens if and only if the
  minimum across all of these
  $(T-1)p + p-s$ numbers is in the
  $p-s$ numbers $\dimp_{k}$ for $k\notin S$. By symmetry, this probability is exactly $\frac{p-s}{(T-1)p + p-s} \le
  \frac{1}{T}$.
  
  As for bounding the probability that
  $\max_{k\in S} \dimp_k \ge \min_{i, t} \dimp_{i}^{(t)}$, 
  we note that by the discussion above, the extra $\dimp_{i, t}$ can be thought of $\dimp_{k}$ for \emph{extra} inactive features
  indexed by $i, t$. The probability that
  $\max_{k\in S} \dimp_k \ge \min_{k\notin S} \dimp_{k}$ can be bounded as in Theorem \ref{thm_additive_monotonic}.
  Taking a union bound with
  $\Pr\left[{\min_{k \notin S} \dimp_k \le
  \min_{i, t} \dimp_{i}^{(t)}}\right]$ proves the theorem's statement.
\end{proof}
\subsection{Experimental Configurations}
All numerical simulations were conducted using Python and the SkLearn packages
to implement both Lasso and the single-depth decision trees. The Lasso
regression was fit using a regularization strength of $C = 2$ and the
\textit{liblinear} solver. Single-depth decision trees splitting point were
identified by the variance reduction as described in the pseudocode of
Algorithm \ref{alg:stump}. Data was generated randomly using the PyTorch
package and all simulations were performed using an NVIDIA RTX A4000 GPU with
Python's CUDA ecosystem.

\fi
\end{document}